\documentclass{svproc_KGT_BS}
\usepackage[T1]{fontenc}
\usepackage{graphicx}
\usepackage{url}

\usepackage[utf8]{inputenc}
\usepackage{graphicx}
\usepackage{amsmath}
\usepackage{amssymb,xcolor}
\usepackage{mathtools}
\usepackage{cancel}
\usepackage{enumitem}
\usepackage[noend]{algorithmic}
\usepackage{algorithm}
\usepackage{wrapfig}
\usepackage{hyperref}
\usepackage{tikz-cd}
\usepackage{tikz}
\usetikzlibrary{automata, positioning}

\usepackage[scriptsize]{subfigure}
\usepackage{epsfig} % for postscript graphics files

\usepackage{enumitem}

\newenvironment{sproof}{%
  \proof}{\endproof}

\def\Re{{\mathbb{R}}}

\def\Nat{{\mathbb{N}}}

\def\histx{{\tilde{x}}}
\def\histX{{\tilde{X}}}

\def\histy{{\tilde{y}}}
\def\histY{{\tilde{Y}}}
\def\histu{{\tilde{u}}}

\def\H{{\mathcal{H}}}
\def\X{{\mathcal{X}}}

\def\is{{\iota}}  % Information state (generic)
\def\his{{\eta}}  % History information state
 
\def\imap{{\kappa}}
\def\fmap{{\phi}}

\def\fmaphist{{\phi_{hist}}}
\def\pihist{{\pi_{hist}}}
\def\V{{\mathcal{V}}}

\def\panic{{\xi}}

\def\atan2{\operatorname{atan2}}
\def\pow{{\rm pow}}

\def\ifs{{\cal I}} %% An information space
\def\ifsder{{\cal I}_{der}} %% History information space
\def\ifshist{{\cal I}_{hist}} %% Nondeterministic information space
\def\imapb{{\kappa_{task}}} %base labeling 
\def\preimg{{^{-1}}}

\DeclareMathOperator*{\id}{id}
\renewcommand{\restriction}{\mathord{\upharpoonright}}

\newcommand{\cat}{{}^\frown}
\newcommand*{\T}{\mathcal{T}}

\usepackage{acronym}
\acrodef{Ispace}[I-space]{\emph{information space}}
\acrodef{Istate}[I-state]{\emph{information state}}
\acrodef{Imap}[I-map]{\emph{information mapping}}
\acrodef{ITS}[ITS]{\emph{information transition system}}
\acrodef{ITSs}[ITSs]{\emph{information transition systems}}
\acrodef{DITS}[DITS]{\emph{deterministic information transition system}}
\acrodef{NITS}[NITS]{\emph{nondeterministic information transition system}}
\acrodef{POMDPs}[POMDPs]{\emph{partially observable Markov decision processes}}
\acrodef{PSRs}[PSRs]{\emph{predictive state representations}}

\acrodef{GNT}[GNT]{Gap Navigation Trees}

\begin{document}
\title{Minimally sufficient structures for information-feedback policies
%\thanks{This work was supported by a European Research Council Advanced Grant (ERC AdG, ILLUSIVE: Foundations of Perception Engineering, 101020977), Academy of Finland (projects PERCEPT 322637, CHiMP 342556), and Business Finland (project HUMOR 3656/31/2019).} 
}
%
%\titlerunning{Abbreviated paper title}
% If the paper title is too long for the running head, you can set
% an abbreviated paper title here
%
\author{Basak Sakcak\and
Vadim K. Weinstein \and Kalle G. Timperi \and
Steven M. LaValle
\thanks{This work was supported by a European Research Council Advanced Grant (ERC AdG, ILLUSIVE: Foundations of Perception Engineering, 101020977), Academy of Finland (projects CHiMP 342556 and BANG! 363637).}
}
\authorrunning{B. Sakcak et al.}
% First names are abbreviated in the running head.
% If there are more than two authors, 'et al.' is used.
%
%\thanks{This work was supported by a European Research Council Advanced Grant (ERC AdG, ILLUSIVE: Foundations of Perception Engineering, 101020977), Academy of Finland (projects PERCEPT 322637, CHiMP 342556)}
\institute{Center for Ubiquitous Computing \\
Faculty of Information Technology and Electrical Engineering \\
University of Oulu, Finland \\
\email{\{firstname.lastname\}@oulu.fi}}
% \institute{
% ABC Institute, Rupert-Karls-University Heidelberg, Heidelberg, Germany\\
% \email{\{abc,lncs\}@uni-heidelberg.de}}
%
\maketitle              % typeset the header of the contribution
\begin{abstract}
%In this paper, we consider planning or control related tasks that require a desirable outcome in the physical world resulting from 
%in the physical world. 
%planning and control problems and characterize structures that are sufficient to solve such problems. 
%We consider 
%In this paper, we consider active tasks which require planning and control to be achieved. Such tasks require a desirable outcome in the physical world resulting from the interactions between an internal system, corresponding to a filter and a policy, and an external system, representing the physical world. 
In this paper, we consider robotic tasks which require a desirable outcome to be achieved in the physical world that the robot is embedded in and interacting with. Accomplishing this objective requires designing a filter that maintains a useful representation of the physical world and a policy over the filter states.  
A filter is seen as the robot's perspective of the physical world based on limited sensing, memory, and computation and it is represented as a transition system over a space of information states. 
To this end, the interactions result from the coupling of an internal and an external system, a filter, and the physical world, respectively,
%two systems: internal and external, which corresponding to a filter, and the physical world, respectively. The coupling between these two systems is established through 
through a sensor mapping and an information-feedback policy. 
%The interactions result from the coupling of these two systems through a sensor mapping and an information-feedback policy. 
Within this setup, we look for sufficient structures, that is, sufficient internal systems and sensors, for accomplishing a given task. We establish necessary and sufficient conditions for these structures to satisfy for information-feedback policies that can be defined over the states of an internal system to exist.
%so that feasible information-feedback policies exist that can be defined over the states of an internal system. 
We also show that under mild assumptions, minimal internal systems that can represent a particular plan/policy described over the action-observation histories exist and are unique. Finally, the results are applied to determine sufficient structures for distance-optimal navigation in a polygonal environment.
\keywords{Planning, Transition Systems, Information Spaces, Sensing Uncertainty, Theoretical Foundations.}
\end{abstract}
\section{Introduction}
%Test for git
%Planning and control is a fundamental problem in robotics and requires determining a sequence of actions which would result in accomplishing a particular task or a set of tasks in a space of environments. 
Determining actions that would cause a robot to accomplish a desired task is a fundamental problem in robotics.
Given a well-defined task structure and a particular robot hardware, solving this problem typically requires designing a filter and a policy over the filter states. 
Therefore, whether filters and policies are designed by engineers, 
computed, or learned, we argue that these two structures should be analyzed and designed together. 

There is a significant difference between a pure inference task, which corresponds to keeping track of the physical world state or a certain aspect of it, and a planning or control task. In the latter case, we are only interested in distinguishing the action to take at a particular state and not the state of the physical world itself. 
%Therefore, it should be sufficient to distinguish only those states that actions are required as opposed to what state the physical world is in.
Therefore, a filter that correctly predicts the outcomes of actions in terms of observations may not be meaningful if all we need is a way to distinguish which action to take.
Most work in the literature separates designing filters %(or observers) 
from respective policies. Typically, filters are designed to estimate the physical world state and policies are determined over this state space
%estimated state 
%of the 
%external 
%physical world 
(see for example, \cite{MajTed17},\cite{ZhuAlo19}). For problems in which the state is not fully observable, \ac{POMDPs} \cite{KaeLitCas98},\cite{RosPinPaqCha08} and \emph{belief spaces} \cite{VitTom11},\cite{AghChaAma14} have been considered for planning. 
Considering directly mapping each observation to an action, early work characterized action-based sensors which provide only the information that is necessary and sufficient, exactly what is needed for an action to be determined \cite{erdmann1995understanding}. 
In this case, states can be grouped through equivalence relations induced by an action-based sensor such that the same action is applied for any state within an equivalence class.
This assumes that at each instant of determining an action, relevant information can be extracted from the environment. Hence, it results in \emph{memoryless} or \emph{reactive} policies. 
However, if this is not the case or if the robot sensors are fixed,
then, the decisions need to be based on the history of previous actions and observations. At one end of the spectrum, assuming unlimited memory, these decisions can be made based on full histories. However, in general, this is computationally unfeasible. 
Therefore, relevant information needs to be extracted from the histories by a filter, allowing a feasible policy to be defined over its states.
%such that over the filter states a feasible policy can be defined. 
%For example, a standard method when it comes to solving control problems corresponds to a filter which estimates the physical-world state and a policy that is described over its states, named \emph{state-feedback policy} (see \cite{MajTed17},\cite{ZhuAlo19}). 

In this work, we analyze the relationship between tasks that require determining actions and a particular filter together with a policy defined over its states. 
Many of the concepts will build upon our previous work
\cite{WeiSakLav22,SakWeiLav22,sakcak2023mathematical}, in which we introduced a general framework built from input-output relationships between two or more coupled dynamical systems. In the basic setting of a robot embedded in an environment, these two dynamical systems correspond to an \emph{internal system} (a centralized computational component) and an \emph{external system} (robot body and the environment). Given particular robot hardware, that is, fixing the robot sensors and actuators, input-output relations correspond to actions and sensor observations. 

The internal system is formally described as a transition system, named an {\ac{ITS}}, with a state space that is an \ac{Ispace}. I-spaces are introduced in \cite[Chapter 11]{Lav06} as a way of analyzing the information requirements of robotic tasks. These were inspired by games with hidden information \cite{basar1998gametheory}. The term \emph{information} is related to the von Neumann-Morgenstern notion of information and not to the later notion introduced by Shannon. 
We see an ITS as a filter and a policy is defined over its states. Derived I-spaces and quotient ITSs are obtained from action-observation histories using \emph{information mappings} (I-maps) that are many-to-one. A derived I-space constitutes the state space of a quotient (derived) ITS.
Within this framework, we analyze conditions that these derived ITSs should satisfy so that feasible, that is, task-accomplishing, policies can be described over their states. 
Planning and control tasks, termed \emph{active tasks}, were already considered in \cite{SakWeiLav22,sakcak2023mathematical}. The results there provided a scaffolding but lacked in establishing the necessary and sufficient conditions and a characterization of sufficient filter-policy pairs, which we do in this paper. In particular, we will consider two cases; fixing the sensor-mapping and analyzing a sufficient ITS, and fixing a particular class of policy and analyzing sufficient sensors for that class. 
%\texttt{Maybe add more here??}
%In particular, we will consider fixing the sensor of a robot and 

There is a limited literature that studied the information requirements for active tasks. This corresponds to determining the weakest notion of sensing or filtering that is sufficient to accomplish a task. 
A notable early work showed, especially for manipulation, that one can achieve certain tasks even in the absence of sensory observations \cite{ErdMas88}. Considering specific problems in mobile robot navigation  \cite{BluKoz78},\cite{TovMurLav07} addressed minimal sensors and filters that are sufficient for navigation. 
In \cite{ZhaShe20}, the authors characterize all possible sensor abstractions that are sufficient to solve a planning problem. 
Closely related to our work, a language-theoretic formulation appears in \cite{SabGhaSheOka19}, in which, Procrustean-graphs (p-graphs) were proposed as an abstraction to reason about interactions between a robot and its environment. Following up from \cite{erdmann1995understanding}, \cite{mcfassel2023reactivity} analyzes conditions for the existence of action-based sensors encoding a particular plan. The authors also propose an algorithm that decomposes those plans for which no action-based sensor exists into subpieces for which one does exist.
  
We focus on sufficient internal systems that can result in task accomplishment once coupled to the external system. 
%This is related to characterizing sufficient ITSs that can carry feasible policies. 
% Therefore, our treatment can be seen as characterizing a sufficient (or minimal) plan, as opposed to characterizing a sufficient structure for planning, for which, a policy need to be computed. 
Therefore, our treatment can be seen as characterizing a sufficient (or minimal) plan, that is, establishing conditions that an ITS should satisfy so that a feasible policy can be defined over its states. 
This is different than characterizing a sufficient ITS for planning, in which case an ITS should allow a policy or a plan to be computed.
A recent work \cite{subramanian2022approximate} addresses the latter case by establishing sufficient conditions that an \ac{Imap} should satisfy for it to allow a dynamic programming formulation.

\section{Information Transition Systems}\label{sec:ITS}

%....?

\begin{figure}[t!]
    \centering
    \includegraphics[width=0.5\linewidth]{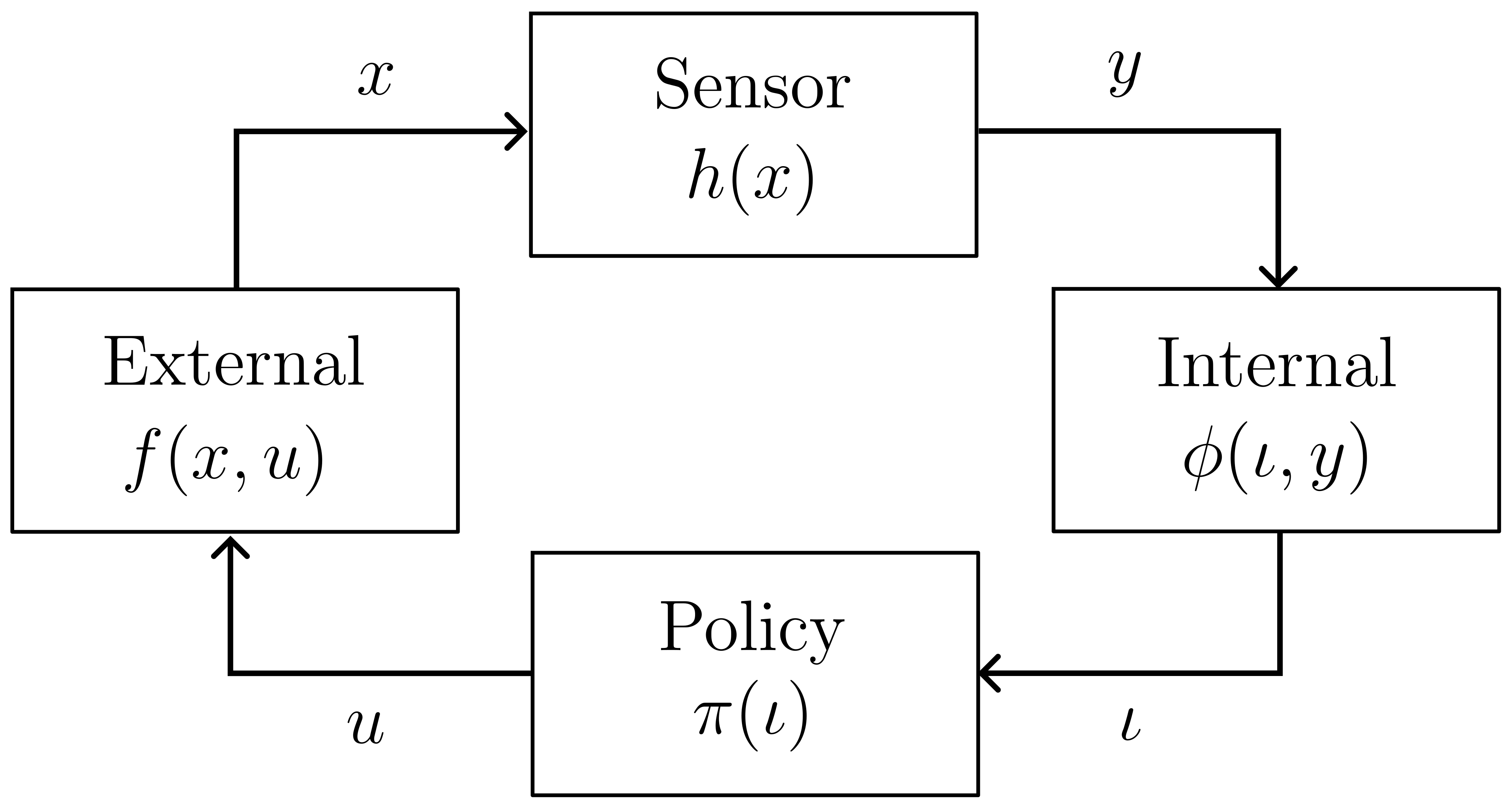}
    \caption{Internal $(\ifs, Y, \fmap)$ and external $(X, U, f)$ systems coupled through coupling functions $h$ and $\pi$, the sensor mapping and the information-feedback policy, respectively.}
    \label{fig:int-ext}
\end{figure}
\vspace{-1em}

%\subsection{Internal and external systems}\label{sec:int_ext_sys}
We consider a robot embedded in an environment and model the interactions between a decision making entity and the physical world as coupled \emph{internal} and \emph{external} systems (see Figure~\ref{fig:int-ext}). 
%The internal system corresponds to the decision making entity and the external system is the physical world, including also the body of the robot. 
The external system corresponds to the totality of the physical environment, including the robot body. 
%Let $X$ denote the set of external states and let $U$ be the set of actions.
%When applied at a state $x \in X$, a control $u \in U$ causes $x$ to change according to a state transition function $f\colon X \times U \to X$. 
%Indeed, an action $u\in U$ refers to the control input to the system and corresponds to the stimuli created by a control command generated by a decision maker. 
%Mathematically, {the external system} can now be expressed as the triple~$(X,U,f)$. 
The internal system %(robot brain) 
%corresponds to 
represents
the perspective of a decision maker. The states of this system correspond to the retained information gathered through the outcomes of actions in terms of sensor observations. To this end, the basis of our mathematical formulation of the internal system is the notion of an \emph{\ac{Ispace}} introduced in \cite[Chapter 11]{Lav06}. Let $\ifs$ be an information space.
%be the set of these %internal 
%information states. 
We will use the term \emph{\ac{Istate}} to refer to the elements of $\ifs$ and denote them by~$\is$. 

We model both the internal system and the external system as transition systems. Following the terminology introduced in \cite{SakWeiLav22}, an internal system will be referred to, more generally, as an \ac{ITS}. 

\vspace{-0.3em}
\begin{definition}[Information Transition System] An ITS is the quadruple
$S=(\ifs, \Lambda, \phi, \is_0)$, 
in which $\ifs$ is an information space corresponding to the states of the transition system, $\Lambda$ is the set of edge labels, 
$\phi \colon \ifs \times \Lambda \to \ifs$
is the information transition function, and $\is_0$ is the initial state. In particular, the edge labels come from one of two sets, namely $\Lambda = U \times Y$ and $\Lambda = Y$, in which $U$ and $Y$ are the sets of actions and observations, respectively. 
\label{def:ITS}
\end{definition} 
\vspace{-1em}

%In particular, we will consider two sets of edge labels, namely $\Lambda = U \times Y$ and $\Lambda = Y$, in which $U$ and $Y$ are the sets of actions and observations, respectively. 
%In \cite{sakcak2023mathematical}, we made a distinction between deterministic and nondeterministic ITSs. However, here we will use the term ITS to refer to a deterministic ITS which will be the object of this paper. Note that a deterministic ITS does not necessarily encode a deterministic model of the external system (see \cite{sakcak2023mathematical} for examples of deterministic ITSs whose model-based I-spaces encode nondeterministic and probabilistic information).

\begin{definition}[State-relabeled ITS] Given an ITS $S=(\ifs, \Lambda, \phi, \is_0)$ a state-relabeled ITS is the 6-tuple $S_\ell = (\ifs, \Lambda, \phi, \is_0, \ell, L)$, in which $\ell: \ifs \rightarrow L$ is a labeling function that attributes to each state $\is \in \ifs$ a unique label from set $L$. 
\end{definition}

The external system is modeled as a state-relabeled transition system as well, for which the set of labels is $Y$.
%for this case labels correspond to the observations.
\begin{definition}[External System] An external system is a state-relabeled transition system $\X_h=(X, U, f, h, Y)$, in which $X$ is the state space, $U$ is the set of edge labels corresponding to the set of actions, $f : X \times U \rightarrow X$ is the (external) state transition function, and $h: X \rightarrow Y$ is a labeling function corresponding to the sensor mapping. 
%In that case, $(X, U, f)$ describes the external system transitions and the labeling function $h : X \rightarrow Y$ is the sensor mapping, leading to a state-relabeled transition system $(X, U, f, h, Y)$.   
\end{definition}

In our framework, a labeling function defined over the states of an ITS has two purposes: (i) deriving quotient ITSs and (ii) acting as a coupling function that maps the output of an ITS to the input of an external system (see Figure~\ref{fig:int-ext}).
%which is also expressed as a state-relabeled transition system. 
The latter case corresponds to the policy $\pi : \ifs \rightarrow U$. %We will discuss this interplay in Section~\ref{sec:suff_structures} in more detail. 

We now focus on the former case of quotient systems derived from a state-relabeled ITS. Let $\ifs_{der}$ be a derived I-space and let $\imap : \ifs \rightarrow \ifsder$ be an \ac{Imap} that is a labeling function defined over the states of an ITS $S=(\ifs, \Lambda, \fmap)$.
Preimages of $\imap$ introduce a partitioning of $\ifs$ creating equivalence classes. 
Let $\ifs/\imap$ be the equivalence classes $[\is]_\imap $ induced by $\imap$ such that $\ifs/\imap=\{[\is]_\imap \mid \is\in \ifs\}$ and $[\is]_\imap=\{\is'\in \ifs \mid \imap(\is')=\imap(\is)\}$. Then, 
through these equivalence classes,
%through the equivalence classes induced by $\imap$, 
we can define a new ITS, called the \emph{quotient of $S$ by $\imap$}, denoted by $S/\imap$,
%It is defined as 
such that
$S/\imap = (\ifs/\imap, \Lambda, \phi/\imap)$, in which\footnote{Here, the map $\phi : \ifs \times \Lambda \to \ifs$ is treated as a subset of $\ifs \times \Lambda \times \ifs$.}
\begin{equation*}
    \phi/\imap :=\big\{\big([\is]_\imap, \lambda, [\is']_\imap\big) \mid (\is, \lambda, \is') \in \phi \big\}.
\end{equation*}
%Above, the map $\phi : \ifs \times \Lambda \to \ifs$ is treated as a subset of $\ifs \times \Lambda \times \ifs$.
% It is defined as $S/\imap = (\ifs/\imap, \Lambda, \Phi/\imap)$, in which
% \begin{equation*}
%     \Phi/\sigma :=\big\{\big([s]_\sigma, \lambda, [s']_\sigma\big) \mid (s, \lambda, s') \in \Phi \big\}.
% \end{equation*}
An important notion when obtaining quotient ITSs through labeling functions is \emph{sufficiency}. 

\begin{definition}[Sufficiency]
A labeling function $\imap \colon \ifs \rightarrow \ifsder$ defined over the states of a transition system $(\ifs, \Lambda, \phi, \is_0)$ is called \emph{sufficient}, if for all $s,t,s',t' \in \ifs$ and all $\lambda \in \Lambda$, the following implication holds:
\begin{equation*}
    \imap(s)=\imap(t) \land s'=\phi(s, \lambda) \land t'=\phi(t, \lambda) \implies %\\
    \imap(s')=\imap(t').  
\end{equation*}
% \begin{equation*}
%     \imap(s)=\imap(t) \land (s,\lambda,s')\in T \land (t,\lambda,t')\in \Phi \implies %\\
%     \imap(s')=\imap(t').  
% \end{equation*}
\end{definition}
In \cite{WeiSakLav22,sakcak2023mathematical}, it was shown that the quotient of an ITS is also an ITS in the sense of Definition~\ref{def:ITS} if and only if $\imap$ is sufficient. Given the label of the current state and the edge label, the label of the next state can be uniquely determined if the labeling function is sufficient, ensuring that the state transitions of the quotient system are deterministic. Hence, $\phi/\imap$ is a function. 
%a history ITS by some $\imap$ is deterministic, so that it is an ITS according to Definition~\ref{def:ITS}, if and only if $\imap$ is sufficient. 

Given a labeling function $\imap$, we might be interested in a finer labeling function which distinguishes the states distinguished by $\imap$ but at a higher resolution. This is achieved by the notion of a refinement of $\imap$ that is defined in the following.
\begin{definition}[\textbf{Refinement of an \ac{Imap}}]
An \ac{Imap} $\imap'$ is a \emph{refinement} of $\imap$, denoted by $\imap' \succeq \imap$, if for all %$\forall
$A \in \ifs/ \imap'$ there exists $B \in \ifs / \imap$ such that $A \subseteq B$.
\end{definition}

The history ITS is a special type of ITS  from which others will be derived through sufficient I-maps. 
%Let $U$ and $Y$ be the sets of possible actions and observations respectively. 
Let $(A)^{<\Nat}$ denote the set of all finite-length sequences of elements of $A$.
The elements of the history information space, denoted by $\ifshist$, are finite sequences of alternating actions and observations which build upon the initial state $\his_0=() \in \ifshist$, therefore, $\ifshist = (U \times Y)^{<\Nat}$. 

\begin{definition}[History ITS] The history ITS $S_{hist}=(\ifshist, U \times Y, \fmaphist, \his_0)$ is an ITS with state space $\ifshist$ and the information transition function $\fmaphist$ %that 
is defined starting from $\his_0 = ()$ through the concatenation operation, that is,
\vspace{-0.3em}
\begin{equation*}
  \his_k = \his_{k-1} \cat (u_{k-1}, y_k).
\end{equation*}
\end{definition}

% We will derive quotient ITSs from the history ITS through I-maps.
% Each \ac{Imap} $\imap$ defined over $\ifshist$
% induces a partition of $\ifshist$ through its preimages, denoted as
% $\ifshist / \imap$.
% \begin{definition}[\textbf{Refinement of an \ac{Imap}}]
% An \ac{Imap} $\imap'$ is a \emph{refinement} of $\imap$, denoted as $\imap' \succeq \imap$, if $\forall A \in \ifshist / \imap'$ there exists a $B \in \ifshist / \imap$ such that $A \subseteq B$.
% \end{definition}
% In \cite{WeiSakLav22,sakcak2023mathematical}, it was shown that the quotient of an history ITS by some $\imap$ is deterministic, so that it is an ITS according to Definition~\ref{def:ITS} if and only if $\imap$ is sufficient. 

%\subsubsection{Coupled internal-external systems}
External and internal systems can be coupled through the coupling functions that map the input of one to the output of the other and vice versa. For us, the sensor mapping $h : X \rightarrow Y$ and the policy $\pi : \ifs \rightarrow U$ are two coupling functions. %(see Section~\ref{sec:int_ext_sys}). 
The coupled system of internal and external described this way is an autonomous system (a closed system), meaning that given an
initial state $(\is_0, x_1) \in \ifs \times X$ there exists a unique trajectory. We denote the function $(\is,x)\mapsto (\is',x')$ by
$\phi*_{\pi,h}f$. 
%which highlights that $\phi*_{\pi,h}f$ is a coupling of $\phi$ and $f$ via the pair of coupling functions $(\pi,h)$. 
Then, the coupled system, denoted by $S_\pi \star \X_h$, is the pair $$S_\pi \star \X_h=(\ifs\times X,\,\phi*_{\pi,h}f).$$
It is also possible to have a single coupling function, in which case the coupled internal-external system admits an input. 
If no policy is determined over the internal system, the coupled system is determined by 
$$S \star \X_h = (\ifs\times X, U, \phi*_{h}f),$$
in which $\phi*_{h}f : \ifs\times X \times U \rightarrow \ifs \times X$ is the state transition function and $U$ is the set of inputs to the coupled system. 
%This can be seen as the perspective that one has for planning purposes which allows to evaluate any possible action at a given I-state. 
This can be seen as the \emph{planner perspective} which allows to evaluate the outcomes of actions at a given I-state. 

Consider the coupling of a history ITS $(\ifshist, U \times Y, \phi_{hist})$ with an external system $(X, U, f)$ through the coupling function $h$, that is, $(\ifshist\times X, U, \phi_{hist}*_{h}f).$
%Consider all finite length input sequences to the coupled system, that is the set $U^{<\Nat}$. 
Let $U^{<\Nat}$ be the set of all finite-length action sequences representing the set of all possible input sequences to the coupled system.
%that can be fed as an input to this coupled system. 
%Let $\R(\ifshist \times X)$ be the set of \emph{reachable states};
A state $(\his, x) \in \ifshist \times X$ of the coupled system is \emph{reachable} from $( (), x_1)$ if there exists some $\histu \in U^{<\Nat}$ such that the state of the coupled system becomes $(\his, x)$ when $\histu$ is applied starting from an initial state $( (), x_1) \in \ifshist \times X$.

\begin{definition}[Set of attainable histories] \label{Def_Attainable_histories}Given a coupling of a history ITS with an
%state-relabeled 
external system $\mathcal{X}_h=(X, f, U, h, Y)$, a history $\his \in \ifshist$ is called \emph{attainable} if there exist $x, x_1 \in X$ such that the state $(\his,x)$ is reachable from an initial state $((),x_1)$. 
We denote by $\ifshist^{\mathcal{X}_h}$ %\subseteq \ifshist$
the set of attainable histories.
%from the coupling of the history ITS with $(X, f, U, h, Y)$.
%, is the set of all histories that are attainable.
%We will denote the set of attainable states from the coupling of the history ITS with $(X, f, U, h, Y)$ by $\ifshist^{E_h}$. 
\label{def:attainable_hist}
\end{definition}
The coupling with an external system $\X_h = (X, f, U, h, Y)$ induces a labeling function $\imap_{att}: \ifshist \rightarrow \{0,1\}$ over the histories through $\imap_{att}\preimg(1)=\ifshist^{\X_h}$. If a history I-state $\his_K$ up to some stage $K$ is unattainable, then any history I-state $\his_N$ up to some stage $N > K$ that builds upon $\his_K$ will also be unattainable. This is stated in the following lemma, which follows directly from Definition~\ref{Def_Attainable_histories}.
%.
%, which follows directly from the definition of attainable histories. 

\begin{lemma}\label{lem:not_attainable} For any $\his \in \imap_{att}\preimg(0)$, and any $(u,y) \in U\times Y$, the next history I-state satisfies %it is true that $\his' =
$\fmaphist(\his, (u,y)) \in \imap_{att}\preimg(0)$.
\end{lemma}

\section{Sufficient Structures for Solving Active Tasks}\label{sec:suff_structures}
We focus on solving \emph{active tasks} which entail executing an information-feedback policy that forces a desirable outcome in the external system.
In \cite{sakcak2023mathematical}, we showed that given a feasible policy defined over the history I-space, minimal ITSs exist that can support that policy.
%represent/support that policy. 
In this section, we will expand on this idea and formally define what supporting a policy means. 
Since our formulation treats ITSs in conjunction with respective policies, we will consider two cases:
\begin{itemize}
    \item Fixing the sensor mapping $h$, which corresponds to fixing $S_{hist}$, and characterizing the sufficient ITSs that support a particular feasible or optimal policy defined over $\ifshist$.
    \item Fixing the particular ITS, $S=(\ifs, Y, \phi)$ and characterizing the sensors $h: X \rightarrow Y$ that are sufficient for this ITS to support a particular class of feasible policies, that is, mappings from $\ifs$ to $U$. In particular, we will consider \emph{reactive policies}, that is, policies that map observations to actions. Hence, the class of ITSs is simply $S=(Y, Y, \id_Y)$, in which $\id_Y$ is the identity function and the policies are of the form $\pi : Y \rightarrow U$. 
\end{itemize}

\subsection{Task Description}\label{sec:task_description}

A task description is encoded through a \emph{task-induced labeling} function $\imapb : \ifshist \rightarrow \{0,1\}$, meaning that $\imapb^{-1}(1)$ is the set of histories that are task accomplishing. A task-induced labeling can be given, learned, or specified through a logical language over $\ifshist$. Tasks can also be described also over $X$ in a similar way. However, in this case the task-induced labeling needs to be determined by a map that checks whether a history I-state satisfies the task description (see \cite{sakcak2023mathematical} for a discussion on the ways of determining $\imapb$).
%% ''Model-free'' case 
For example, when tasks are specified using a logical language over $\ifshist$, the resulting sentences of the language involve combinations of predicates that assign truth values to subsets of $\ifshist$. 
This implicitly defines $\imapb : \ifshist \rightarrow \{0,1\}$, in which $0$ stands for \texttt{false} and $1$ stands for \texttt{true}.
%% ''Model-based'' case. 
On the other hand, when a task description is determined as a logical language over $X$, the resulting sentences of the language involve combinations of predicates that assign truth values to subsets of $X$ (see \cite{FaiGirKrePap09,belta2007symbolic} for examples using linear temporal logic). 
%If a task is defined over $X$ then the sentence satisfiability of a history I-state must be determined by an I-map that converts history I-states into expressions over $X$. 

Given a sequence of actions $\histu=(u_1, u_2, \dots, u_{N})$ and an initial state $x_1$, an external state trajectory $x_1 \diamond \histu$ is a sequence of external system states defined as
\begin{equation*}
   x_1 \diamond \histu = (x_1, x_2=f(x_1,u_1), \dots, x_{N+1}=f(x_{N},u_{N})).
\end{equation*}
Under a sensor-mapping $h$, an external system trajectory $x_1\, \diamond\, \histu$ corresponds to a unique action-observation history, that is 
\begin{equation*}
(h(x_1), u_1, h(x_2), \dots, u_N, h(x_{N+1})).
\end{equation*}
Note that the inverse is not necessarily true, since the sensor mapping $h : X \rightarrow Y$ is not necessarily invertible (it is not one-to-one). 
Thus, the same action-observation history can lead to different external system trajectories depending on the particular initial state $x_1 \in X$. 
%Or if there are disturbances, then depending also on the particular realization of those but so far this is not included in the model so I think better to skip for now. 
In this case, it may not be possible to determine whether a given history satisfies the task description, as this depends on the particular sensing and actuation setting in relation to the task description given over $X$. 
In this work, when considering tasks defined over $X$, we will define the respective $\imapb$ in the following way. Let $g : \ifshist \rightarrow \pow(\tilde{X})$ be a function that maps a history to the corresponding set of possible external system trajectories. 
Then, whether a history $\his$ satisfies a task description given over $X$ is determined based on whether all $\histx \in g(\his)$ satisfy the task description. 
%For example, if the task is to reach a goal region $X_G \in X$, a goal detector (a sensor that reports if the state $x \in X_G$) allows to 

%Accomplishing an active task requires that a sentence of interest becomes true as a result of the execution of a policy defined over the internal system states, that is, the resulting history $\his$ belongs to $\imapb\preimg(1)$. 
Accomplishing an active task requires that the resulting history $\his$ belongs to $\imapb\preimg(1)$. 
%To this end, we will assume that the coupled 
Let $S_\pi = (\ifs, Y, \fmap, \pi, U)$ and $\X_h=(X, U, f, h, Y)$ be a policy-labeled ITS and an external system, respectively. The corresponding coupled system is $S_\pi \star \X_h = (\ifs \times X, \phi \star_{\pi,h} f)$. 
We consider tasks that are defined over finite-length histories for which the satisfaction of sentences can be determined in finite time 
\footnote{For a discussion on infinitary tasks and how they can be transcribed as tasks over finite-length histories see Section~4.1 in \cite{sakcak2023mathematical}.}. 
We will consider tasks that have a termination condition so that once a sentence becomes true, the interaction $S_\pi \star \X_h$ stops resulting in the history $\his_N \in \imapb\preimg(1)$ for some $N$. 

\begin{definition}[Feasible Policy]\label{def:feasible_policy} A policy $\pi: \ifs \rightarrow U$ defined over the states of $S=(\ifs, Y, \fmap)$ is \emph{feasible} if for all $x_1 \in X$ the coupled system $S_\pi \star \X_h$ initialized at $(\is_0, x_1)$ results in $\his_N \in \imapb\preimg(1)$, in which $N$ may depend on $(\is_0, x_1)$.
\end{definition}

In this definition, we have assumed for simplicity that task accomplishment can be achieved for any initial external state $x_1 \in X$. This may not be true in general. In case it is not, a feasible policy should be defined for all $x_1 \in X' \subseteq X$, in which $X'$ is the set of states for which there exists a task-accomplishing history.
%\texttt{It says for all $x_1 \in X$ but should be the ones that are reachable only.. say a few words on this}

%% Missing
% \texttt{
% \begin{itemize}
%     \item Different state trajectories can lead to the same observation history.
%     \item How to describe task for those: worst case, best case, probabilistic (if statistics are available) formulations?
%     \item Worst case: mark task achieving only those histories that $g(\his) \subset \text{`task achieving state trajectories'}$
%     \item Describability (is this even a word??) of a task? 
%     This relates to sufficient sensor (or to Vadim's strategic sufficiency).. how to strategically mention this?
%     %\item Couple the external system with the history ITS (using \eqref{eqn:coupled_sys}). Then, certain histories are not achievable (they will never be realized). 
% \end{itemize} 
% }

\subsection{ITSs Sufficient for Feasible Policies}\label{sec:suff_ITSs}
%Section~\ref{sec:task_description} formalized the task description over histories. 
% Given a task description over $\ifshist$, accomplishing an active task entails determining an ITS and a policy over its states so that the resulting histories of the coupled internal-external systems satisfy the task description. Formally this means that the resulting histories belong to $\imapb\preimg(1)$.
%the preimage of $1$ under $\imapb$. 

In this section, we derive conditions that a state-relabeled ITS $(\ifs, Y, \phi, \is_0)$ needs to satisfy in order to support %(express) 
a particular feasible policy $\pi_{hist}:\ifshist \rightarrow U$.
%(plan) 
%described over $\ifshist$. %histories. 
To achieve this objective, we first define the restriction of a history ITS by a $\pi_{hist}$ which drops the dependency on actions in transitions. 
%Consequently, the quotient system is required to distinguish the actions distinguished by the policy.
%To this end, we first define the restriction of a history ITS by a policy $\pi_{hist}$. This will correspond to the ITS $(Y^{<\Nat}, Y, \phi_{Y},())$ whose edge labels are elements of $Y$ instead of $U\times Y$ and whose states are relabeled by some labeling function $\imap_\pi$. 
%We will then define an I-map $\imap$ that can be used to derive a new ITS $(\ifs, Y, \phi, \is_0)$ as the quotient of $(Y^{<\Nat}, Y, \phi_{Y}, ())$ by $\imap$ so that $(\ifs, Y, \phi)$ supports the policy $\pi_{hist}$. 

Consider the history ITS $(\ifshist, U\times Y, \phi_{hist}, ())$ coupled to the external system $\X_h=(X, f, U, h, Y)$. This leads to the set of attainable histories $\ifshist^{\X_h}$ and the corresponding labeling function $\imap_{att}$, see Definition~\ref{Def_Attainable_histories}.
%Given $\ifshist^{\X_h}$, we may define a policy %can be defined as
%$\pi_{hist} : \ifshist^{\X_h} \rightarrow U$ that maps attainable histories to actions. 
%The policy $\pi_{hist}$ further restricts the set of all histories to only those that are %can only be
%achieved by applying it. %under $\pi_{hist}$.
Then, any policy $\pi_{hist} : \ifshist^{\X_h} \rightarrow U$ %defined on $\ifshist^{\X_h}$ 
further restricts the set of all histories to the subset that can be realized following $\pi_{hist}$.
%to the subset containing those that are achieved by applying it. 
Due to $\pi_{hist}$, at each history I-state $\his$, only a single action is possible.
We will call this the \emph{restriction of $\ifshist^{\X_h}$ by $\pi_{hist}$}, denoted by $\ifshist^{\X_h}\restriction\pi_{hist}$, that is,
\begin{equation}\label{eqn:policy_restrict}
\ifshist^{\X_h}\restriction\pi_{hist} := \big\{  \his \cat (u,y) \in \ifshist^{\X_h} \mid (\his, y) \in \ifshist^{\X_h}\times Y \land u = \pi_{hist}(\his)  \big\}.
\end{equation}
%Notice that $\pi_{hist}$ also restricts the transitions to only those that can be realized following this policy. 
Following $\pi_{hist}$ restricts also the transition function $\fmaphist$. 
Let 
\begin{equation} \label{Eq_phi_hist_prime}
    \fmaphist':=\big\{ (\his, (u,y) , \his') \in \fmaphist \mid \his' \not\in \ifshist^{\X_h}\restriction\pi_{hist} \big\}
\end{equation}
be the set of transitions that cannot be realized following $\pi_{hist}$.\footnote{For notational convenience, we treat $\fmaphist$ in~\eqref{Eq_phi_hist_prime} as a subset of $\ifshist \times (U \times Y) \times \ifshist$.} Then, the set of transitions achievable under $\pi_{hist}$ is simply the set difference %defined as
\begin{equation}\label{eqn:rest_trans_set}
    \fmaphist\restriction_{\pi_{hist}} := \fmaphist \setminus \fmaphist'.
\end{equation}
Notice that $\fmaphist\restriction_{\pi_{hist}}$ describes a function with domain $\ifshist^{\X_h}\restriction\pi_{hist} \cup \imap_{att}\preimg(0)$. 

Let $U^\panic := U \cup \{\panic\}$, in which $\xi$ serves as a dummy label indicating that an I-state is not attainable. We encode this by defining a labeling function $\imap_\pi: \ifshist^{\X_h}\restriction\pi_{hist} \cup \imap_{att}\preimg(0) \rightarrow U^\panic$ %\cup \{\panic\}$
via
\begin{equation}\label{eqn:k_pi}
    \imap_\pi(\eta) := 
    \begin{cases}
        \pi_{hist}(\his)& \text{if } \imap_{att}(\his)=1\\
        \panic &\text{otherwise.}
    \end{cases}
\end{equation}
Note that $\imap_\pi$ encodes both the labeling $\imap_{att}$, which distinguishes unattainable histories, and the policy $\pi_{hist}$, which distinguishes histories in terms of actions to take.
%Let $\imap_\pi$ be a labeling function defined by %\texttt{->$\panic$ is the panic state.} 

\begin{lemma}\label{lem:u_for_y}Let $(\his, (u,y), \his'), ( \his, (u', y), \his'' )  \in \fmaphist\restriction_{\pi_{hist}}$. Then $\imap_{\pi}(\his')=\imap_{\pi}(\his'')$.
\end{lemma}
\begin{proof}
Suppose $\his \in \imap_{att}\preimg(1)$. Then $u=u'=\pi_{hist}(\eta)$ and since for all $y$ there is a unique $\his'$ by the construction of $\fmaphist\restriction_{\pi_{hist}}$ (see Eq.\eqref{eqn:rest_trans_set}) it follows that $\his'=\his''$. Suppose $\his \in \imap_{att}\preimg(0)$. Then, for any $(u,y) \in U \times Y$, the resulting $\his', \his''$ satisfy $\his', \his'' \in \imap_{att}\preimg(0)$ due to Lemma~\ref{lem:not_attainable}. Therefore, $\imap_{\pi}(\his') = \imap_\pi(\his'') = \panic$. %\texttt{Maybe should add this around Definition~\ref{def:attainable_hist}.}
\qed
\end{proof}
Let $\imap_{\histY}: \ifshist \rightarrow Y^{<\Nat}$ be an I-map that maps each action-observation history $\his=(y_1, u_1, \dots, u_{N-1}, y_{N})$ to the corresponding observation history $\histy=(y_1, \dots, y_{N})$.  
\begin{lemma}\label{lem:refinement_kappaY} The I-map $\imap_\histY$ with its domain restricted to $\ifshist^{\X_h}\restriction\pi_{hist} \cup \imap_{att}\preimg(0)$ is a refinement of $\imap_\pi$, that is, $\imap_\histY \succeq \imap_{\pi} $.
\end{lemma}
\begin{proof}
%We need to show that for all $A \in \ifshist^{\X_h}\restriction\pi_{hist} \cup \imap_{att}\preimg(0)/\imap_Y$ there exists a $B \in \ifshist^{\X_h}\restriction\pi_{hist} \cup \imap_{att}\preimg(0)/\imap_{\pi}$ such that $A \subseteq B$. 
Let $[\histy]_{\imap_\histY}$ be an equivalence class induced by $\imap_\histY$. By Lemma~\ref{lem:u_for_y}, the preimage of $\imap_\histY\preimg(\histy)$ is either a singleton or it satisfies $\imap_\pi(\his)=\panic$ for all $\his \in \imap_\histY\preimg(\histy)$. This proves that for all $A \in \ifshist^{\X_h}\restriction\pi_{hist} \cup \imap_{att}\preimg(0)/\imap_\histY$ there exists a $B \in \ifshist^{\X_h}\restriction\pi_{hist} \cup \imap_{att}\preimg(0)/\imap_{\pi}$ such that $A \subseteq B$. 
\qed
\end{proof}

%By Lemma~\ref{lem:u_for_y} and \ref{lem:refinement_kappaY}, 
Thanks to Lemma~ \ref{lem:refinement_kappaY}, we can define the restriction of the history ITS by $\pi_{hist}$ as a quotient of history ITS by $\imap_\histY$, that is, $S_{hist}/\imap_\histY$, together with a labeling function $\pi: Y^{<\Nat} \rightarrow U^\panic$ which is defined through $\imap_\pi$. Furthermore, due to Lemmas~\ref{lem:u_for_y} and \ref{lem:refinement_kappaY}, the I-state transitions need to depend only on elements of $Y$. Therefore, we will define the set of transitions $\fmap_\histY$ by taking the projection of $\fmaphist'$ onto $Y^{<\Nat}\times Y \times Y^{<\Nat}$.

\begin{definition}[Restriction of history ITS by $\pi_{hist}$] The restriction of a history ITS by $\pi_{hist}$ is the state-relabeled ITS, $\mathcal{S}_{hist}\restriction_{\pi_{hist}}=(Y^{<\Nat}, Y, \phi_\histY, (), \pi, U^{\panic})$ such that under $\pi : Y^{<\Nat} \rightarrow U^\panic$, $\histy \mapsto \imap_{\pi}(\his)$, in which $\his \in \imap_Y\preimg(\histy)$. Note that for all $\his, \his' \in \imap_\histY\preimg(\histy)$, $\imap_\pi(\his)=\imap_\pi(\his')$ (Lemma~\ref{lem:u_for_y}).
\end{definition}

\begin{lemma}\label{lem:fullness_restriction} The ITS corresponding to the restriction of the history ITS by $\pi_{hist}$, that is, $(Y^{<\Nat}, Y, \phi_\histY)$ is full\footnote{A transition system
    $(S,\Lambda,T)$ is called \emph{full}, if $\forall s\in S, \lambda \in \Lambda$
    there exists at least one $s'\in S$ with $(s,\lambda,s')\in T$.}. 
\end{lemma}

%This is due to Lemma~\ref{lem:refinement_kappaY} since $\imap_Y$ is a sufficient refinement of $\imap_{\pi}$ we can determine the labels attributed to histories by $\imap_\pi$. 
%Furthermore, based on Lemma~\ref{lem:u_for_y}, we will define the restriction of history ITS by taking the projection of $\fmaphist'$ onto $Y^{<\Nat}\times Y \times Y^{<\Nat}$, denoted by $\fmap_Y$.  
%\texttt{Define a $\pi'$ from $Y^{<\Nat}$}

Let $\Pi((\ifs, Y, \phi), \mu) : Y^{<\Nat} \rightarrow U^\panic$ 
%U \cup \{ \panic\} $ 
be a function in which $(\ifs, Y, \phi)$ is an ITS and $\mu: \ifs \rightarrow U^\panic$ is a labeling function. Given an input sequence of some length $N$, that is, $\tilde{y}=(y_1, y_2, \dots, y_N)$, $\histy \mapsto u_N$ under $\Pi((\ifs, Y, \phi), \mu)$, in which, $u_N$ is the last element of the respective output sequence $\histu = (\mu(\is_1), \mu(\is_2), \dots, \mu(\is_N))$ with $\is_i=\phi(\is_{i-1},y_i)$ for $i=1,\dots,N$.

\begin{definition}[Supports $\pi_{hist}$] Let $\big(Y^{<\Nat}, Y, \phi_\histY, (), \pi, U^\panic\big)$ be the restriction of the history ITS $(\ifshist, U\times Y, \phi_{hist}, ())$ by a policy $\pi_{hist}$.
Let $S=(\ifs, Y, \phi)$ be an ITS. $S$ supports $\pi_{hist}$ if there exists $\mu : \ifs \rightarrow U$ such that $\Pi(\ifs, Y, \phi, \mu) = \pi$. 
\end{definition}

% Recall that $\imapb$ is a task labeling that categorizes histories to those that satisfy the task description, that is, the set $\imapb\preimg(1)$, and those that do not, that is, the set $\imapb\preimg(0)$.
% \begin{definition}[Feasible Policy]\label{def:feasible_policy_hist}
% A policy $\pi_{hist}: \ifshist^{\X_h} \rightarrow U$ is a feasible policy if and only if ...
%A policy $\pi_{hist}: \ifshist^{\X_h} \rightarrow U$ is a feasible policy if and only if $\ifshist^{\X_h}\restriction\pi_{hist}$ satisfies that $\ifshist^{\X_h}\restriction\pi_{hist} \subseteq \imapb\preimg(1)$. 
%\end{definition}

The following theorem establishes a necessary and sufficient condition for an ITS $(\ifs, Y, \fmap, \is_0)$ for it to support the feasible policy $\pi_{hist}$.

\begin{theorem}\label{thm:nec_suff_cond_ITS}
Let $\big(Y^{<\Nat}, Y, \phi_\histY, (), \pi, U^\panic\big)$ be the restriction of the history ITS by a feasible policy $\pi_{hist}$.
An ITS $S=(\ifs, Y, \fmap, \is_0)$ supports $\pi_{hist}$ if and only if $S$ is the quotient of $\big(Y^{<\Nat}, Y, \phi_{\histY}, ()\big)$ by some sufficient $\imap$ satisfying $\imap \succeq \pi$.
\end{theorem}
\begin{proof}
\textbf{$\implies$ direction (If $\imap \not\succeq \pi$, then there does not exist a $\mu$):} 
Suppose $\imap$ is not a refinement of $\pi$. Then, there exist $\is \in \ifs$ and $u \in U^\panic$ such that $\imap\preimg(\is) \setminus \pi\preimg(u)\neq \emptyset$. 
This implies that there exist $\histy$ and $\histy'$ such that $\imap(\histy)=\imap(\histy')$ and $\pi(\histy) \neq \pi(\histy')$. Then there does not exists a $\mu$ that satisfies $\Pi((\ifs, Y, \phi), \mu) = \pi$ since input sequences $\histy$ and $\histy'$ cannot be distinguished by $\imap$.

\textbf{$\impliedby$ direction (If $\imap \succeq \pi$, then there exists a $\mu$):} We will prove this by construction. Since $\imap$ is a refinement of $\pi$, every set in $Y^{<\Nat}/\imap$ is a subset of $Y^{<\Nat}/\pi$ which implies that for all $\is \in \ifs$ it is true that for each $\histy,\histy' \in \imap\preimg(\is)$, $\pi(\histy) = \pi(\histy')$. Then, there exists a function $\mu: \ifs \rightarrow U^\panic$ such that $\mu(\imap(\histy))= \mu(\imap(\histy'))$ since $\pi(\histy) = \pi(\histy') $ for all $\histy, \histy' \in \imap\preimg(\is)$.  
\qed
\end{proof}

Clearly, $\imap=\pi_{hist}$ satisfies this condition. 
However, in most cases it is not computationally feasible to define policies over entire histories. Therefore, we look for a minimal $\imap$ that satisfies this condition. This corresponds to finding a minimal sufficient refinement of $\pi$ that gives out the minimal ITS that can support $\pi_{hist}$ (see also Theorem 3 in \cite{sakcak2023mathematical}). This is stated in the following result.

\begin{corollary} Let $\big(Y^{<\Nat}, Y, \phi_Y, (), \pi, U^{\panic}\big)$ be the restriction $S_{hist}\restriction_{\pi_{hist}}$ and let $\bar{\pi}$ be a minimal sufficient refinement of $\pi$. A minimal ITS that supports $\pi_{hist}$ is the quotient of $(Y^{<\Nat}, Y, \phi_\histY, ())$ by $\bar{\pi}$.  
%The minimal ITS that supports $\pi_{hist}$ is the quotient of $(Y^{<\Nat}, Y, \phi, ())$ by the minimal sufficient refinement of $\pi$ such that the restriction $S_{hist}\restriction_{\pi_{hist}}$ coincides with $\big(Y^{<\Nat}, Y, \phi_Y, (), \pi, U^{\panic}\big)$.
%is the restriction of history ITS by $\pi_{hist}$.
Furthermore, this minimal ITS is unique.
\end{corollary}
\begin{proof}
Uniqueness follows from Theorem~4.19 in \cite{WeiSakLav22} %\texttt{(double check this)} 
which states that the minimal sufficient refinement of a labeling function defined over the states of a transition system is unique if the transition system is full. By Lemma~\ref{lem:fullness_restriction}, $S_{hist}\restriction_{\pi_{hist}}$ is full.
\qed
\end{proof}

\begin{corollary} Suppose that $\pi_{hist}$ is a feasible policy and $(\ifs, Y, \phi, \is_0, \mu, U^\panic)$ is an ITS that supports $\pi_{hist}$. Then, the labeling function $\mu : \ifs \rightarrow U^\panic$ is a feasible policy as defined in Definition~\ref{def:feasible_policy}.
% then $\mu : \ifs \rightarrow U^\panic$ %a policy labeled ITS 
% defined as a labeling function for
% $(\ifs, Y, \phi, \is_0, \mu, U^\panic)$ which supports $\pi_{hist}$ is a feasible policy as defined in Definition~\ref{def:feasible_policy}
\end{corollary}

\begin{figure}[t!]
\centering
\subfigure[]{\includegraphics[width=.2\linewidth]{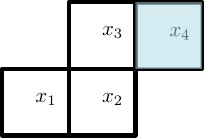}} \hspace{1.5em}
\subfigure[]{\includegraphics[trim={4cm 0 0 0},clip, width=.4\linewidth]{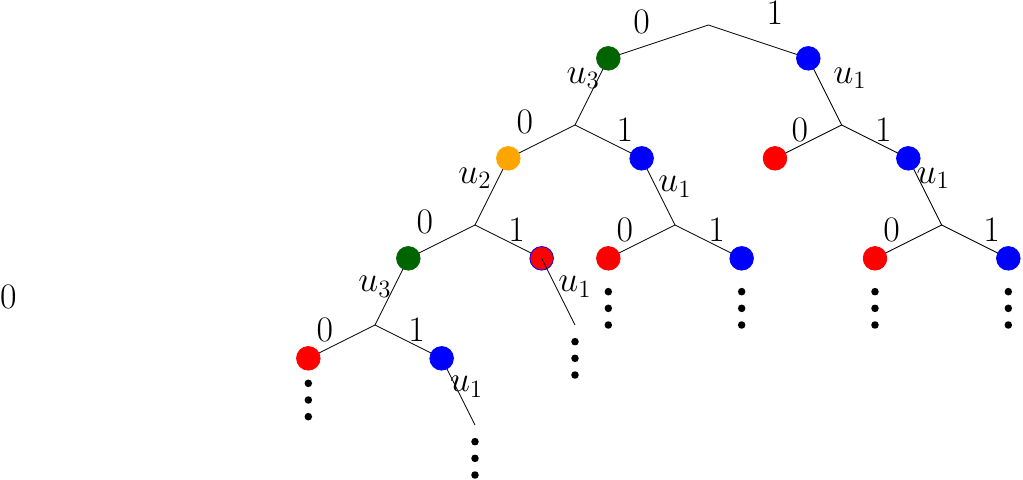}} \hspace{1.5em} 
\subfigure[]{\includegraphics[width=.25\linewidth]{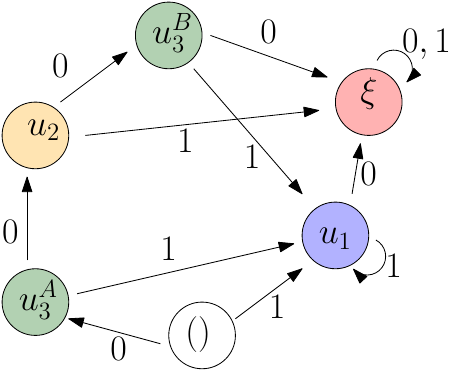}} %\hspace{2em}
    \caption{(a) Skewed tetromino environment. (b)Labels attributed by $\imap_\pi$ with labels $u_1$ (blue), $u_2$ (orange), $u_3$ (green), and $\xi$ (red). (c) $S$ is the quotient of $S_{hist}\restriction_{\pi_{hist}}$ by $\imap$.}
    %\vspace{-1.8em}
\label{fig:tetromino}
\end{figure}

The following example illustrates the introduced concepts.
\begin{example}[Skewed Tetromino Environment]\label{ex:tetromino}
Consider the Skewed Tetromino environment given in Figure~\ref{fig:tetromino}(a). The state space is $X=\{x_1, x_2, x_3, x_4\}$ with a sensor mapping $h(x) = 1$ for $x=x_4$ and $h(x)=0$, otherwise. The action space $U=\{u_1, u_2, u_3\}$ such that $u_1, u_2, u_3$ correspond to stopping, moving one cell up and moving one cell to the right, respectively. The task is defined as reaching the state $x_4$ which corresponds to obtaining $y_k = 1$ for some $k$. Figure~\ref{fig:tetromino}(b) shows the labeling determined by $\imap_\pi$ that distinguishes unattainable histories and histories labeled by $\pi_{hist}$. From $\kappa_\pi$ we obtain the restriction of history ITS, that is, $\mathcal{S}_{hist}\restriction_{\pi_{hist}}=(Y^{<\Nat}, Y, \phi_Y, (), \pi, U^{\panic})$. Notice that $\pi$ is not sufficient. Consider observation histories $\histy = (0)$ and $\histy'=(0, 0, 0)$ which satisfy $\pi(\histy) = \pi(\histy') = u_3$, however, $\pi(\histy \cat 0) \not= \pi(\histy'\cat 0)$. Figure~\ref{fig:tetromino}(c) shows the quotient $S=(\ifs, Y, \phi, \is_0)$ of $\mathcal{S}_{hist}\restriction_{\pi_{hist}}$ by $\imap$ which is the minimal sufficient refinement of $\pi$. This is the minimal ITS that can support $\pi_{hist}$. There exists a policy $\mu$ with 
\begin{equation}
    \mu(\is) = 
    \begin{cases}
        u_3 & \text{if } \is \in \{ u^A_3, u^B_3 \} \\
        \is &\text{otherwise},
    \end{cases}
\end{equation}
such that $S_\mu \star \X_h$ accomplishes the task. 
% \texttt{%If space permits.. 
% \begin{itemize}
%     %\item Draw the figure 
%     %\item Bumper sensor or goal sensor (check tetromino 1,2 in figures folder).. Goal sensor is better!
%     \item Label histories with panic and actions 
%     \item Unique refinement 
%     \item A bijection from the nondeterministic ITS 
% \end{itemize}}
\end{example}

Let $\ifs_{ndet} \subseteq \pow(X)$ be a nondeterministic I-space and consider the ITS $S^{ndet} = (\ifs_{ndet}, Y, \phi_{ndet}, X)$. Let $\hat{X}(A, y) = \{ f(x, \mu'(A)) \mid x \in A\}$ and define $\phi_{ndet}(A, y) := \hat{X}(A, y) \cap h\preimg(y)$. Furthermore, define 
\begin{equation*}
    \mu'(\is) = 
    \begin{cases}
        () & \text{if } \is = X \\
        u_1 & \text{if } \is = \{ x_4 \} \\
        u_2 & \text{if } \is = \{x_2\} \\
        u_3 &\text{otherwise}.
    \end{cases}
\end{equation*}
\begin{proposition}\label{prop:isomorph_ndet} $S^{ndet}$ is isomorphic to $S$ defined in Example~\ref{ex:tetromino}.
\end{proposition}
\begin{proof}
There exists a mapping $\psi$ that maps the states of $S^{ndet}$ to those of $S$ such that $X \mapsto ()$, $\{x_1, x_2, x_3\} \mapsto u_3^A$, $\{x_2\} \mapsto u_2$, $\{ x_3\} \mapsto u_3^B$, and $\emptyset \mapsto \xi$. Then, by checking the definitions, for any $A, B \in \ifs_{ndet}$ and any $\is, \is' \in \ifs$ if $\psi(A)=\is$ and $\psi(B)=\is'$ it is also true that $\phi(\is)=\is'$ implies $\phi_{ndet}(A)=B$.
\end{proof}
From Proposition~\ref{prop:isomorph_ndet}, it can be deduced that $S^{ndet}$ is a minimally sufficient ITS that can support $\pi_{hist}$ which takes the external system state to $x_4$. This is related to the \emph{Good Regulator Theorem} \cite{ConAsh70} which roughly states that for any policy to regulate the external system state to a target value, the internal system should encode a model of the external system.

\subsection{Multiple Policies}
So far we have considered a single task and looked for an ITS that can support a feasible policy for that task. 
However, typically, robots are expected to achieve multiple tasks. Therefore, in this section we will focus on ITSs that can support a set of feasible policies satisfying a set of tasks. 

Let $\T = \{\imap^{T_i}_{task}\}_{i=1,\dots,N}$ be a set of $N$ number of tasks such that each task $T_i$ induces a labeling function $\imap^{T_i}_{task}$. 
Then, for each task, there exists a feasible policy $\pi^{T_i}_{hist} : \ifshist \rightarrow U^\panic$ defined over the history \ac{Ispace}.  

%\texttt{Say something about how each task can also mean a different environment. Therefore, changing the labeling $\imap_{att}$ and also the policy...}

% \begin{itemize}
% \item For each task there exists a $\pi^{T_i}_{hist}$
% \item $S = (\ifs, Y, \phi)$ should satisfy that it supports $\pi^{T_i}_{hist}$ for $i=1, \dots, N$.
% \item Each $\pi^{T_i}_{hist}$ determines an Imap $k^{T_i}_\pi$ defined as in Eqn.~\ref{eqn:k_pi} 
% \item Let $\imap^\T_\pi$ be the common refinement of all $k^{T_i}_\pi$ 
% \item Let $S_{hist}\restriction_{\Pi^\T_{hist}}=(Y^{<\Nat}, Y, \phi_Y, (), \pi^{\T}, U^\panic)$ be the restriction of history ITS by the set of policies $\pi$
% \end{itemize}

% Let $S = (\ifs, Y, \phi)$ be an ITS and let $\mu_i : \ifs \rightarrow U$ be a feasible policy satisfying task $T_i$ such that $S_{\pi_i} \star \X_h$ achieves task $T_i$. 

%The following result follows from Theorem~\ref{thm:nec_suff_cond_ITS}.
\begin{theorem}\label{thm:nec_suff_mult_pols}
Let $S_{hist}\restriction_{\pi^{T_i}_{hist}}=(Y^{<\Nat}, Y, \phi_\histY, (), \pi_i, U^\panic)$ be the restriction of history ITS by a feasible policy $\pi^{T_i}_{hist}$ for $i=1,\dots,N$. 
Let $\pi^\T$ be the join (least upper bound) of $\{ \pi_i\}_{i=1,\dots,N}$.
An ITS $S=(\ifs, Y, \fmap, \is_0)$ supports $\pi^{T_i}_{hist}$ for all $i=1,\dots,N$, if and only if $S$ is the quotient of $(Y^{<\Nat}, Y, \phi_{\histY}, ())$ by some sufficient $\imap$ satisfying $\imap \succeq \pi^\T$.
\end{theorem}
\begin{proof}
The set of policies defined over $Y^{<\Nat}$ forms a lattice. Hence $\pi^\T$ exists and is unique. Because $\pi^\T$ is the join of $\{ \pi_i\}_{i=1,\dots,N}$, $\imap \succeq \pi^\T$ implies $\imap \succeq \pi_i$ for any $i = 1, \dots, N$. The rest of the proof follows from Theorem~\ref{thm:nec_suff_cond_ITS}.
\qed
\end{proof}

\subsection{Sufficient Sensors for Reactive Policies}

In the previous section, we fixed the sensor mapping $h : X \rightarrow Y$ and looked for an ITS that can support a particular feasible policy defined over the action-observation histories. In this section, we fix the form of the ITS and its dependence
%of the ITS %characterized by
on the sensor mapping $h$, and leave $h$ free. This implies that the task-induced labeling $\imap_{task}$ also depends on the selected sensor mapping. 
%Therefore, we will consider tasks defined over the external system trajectories.

Let $\H$ be the set of all sensor mappings, or equivalently, the set of all partitions of $X$. In the following, $Y$ is seen as the set of labels attributed to the subsets forming the partition (see \cite{Lav12b}). Therefore, the particular range of $h$ does not carry any importance as long as it induces the same partitioning.
We fix the ITS to be of the form $S=(Y,Y, \id_Y, \is_0)$, in which $\id_Y$ is the identity function, and fix the policy $\pi_Y : Y \rightarrow U$ leaving $Y$ free, that is, leaving $h$ free. This results in a set of ITSs and respective reactive policies characterized by the sensor mapping $h$. With reactive policy we mean a policy that maps the each observation to an action, hence, same observation is always mapped to the same action. 
% We then look for a sensor mapping $h \in \H$ such that $h : X \rightarrow Y$ ensures that the resulting set of histories from the coupling of $S_{\pi_Y}$ with the external system $\X_h$ through coupling functions $h$ and $\pi_Y$, that is, $S_{\pi_Y}\star \X_h$, satisfy the task description. 
% %, that is, it belongs to the set $\imapb\preimg(1)$. 
% A sensor that satisfies this condition is called \emph{sufficient for a reactive policy}. 
Let $S_{\pi_Y} = (Y, Y, \id_Y, \is_0, \pi_Y, U)$ and $\X_h=(X, U, f, h, Y)$ be an ITS labeled with the policy $\pi_Y$ and an external system with a sensor mapping $h$, respectively. 
Their coupling $S_{\pi_Y} \star \X_h$ initialized at $(\is_0, x_1)$ results in the following action-observation history 
\begin{equation*}
    \his_N = \left( h(x_1) , \left(u_1, h(x_2) \right), \dots \left(u_{N-1}, h(x_N) \right) \right),
\end{equation*}
in which each $x_{i+1} = f (x_i, u_i)$ and $u_i=\pi_Y \circ \id_Y \circ \, h(x_i)$.

%\texttt{Maybe it makes sense to explain here how task description, that is $\imapb$, now depends on $h$ and add external system history?? or should it be already back in task description.. }

\begin{definition}[Sufficient sensor for a reactive policy] A sensor mapping $h$ is called \emph{sufficient for a reactive policy} if any initialization $(\is_0, x_1)$ of $S_{\pi_Y} \star \X_h$ results in a history $\his_N \in \imapb\preimg(1)$ for some $N$. 
\end{definition}
Consider the strongest possible sensor $h_{bij} : X \rightarrow Y$, i.e.~a bijection. For notational simplicity, without loss of generality, we will assume $Y=X$ so that $h_{bij}$ is the identity function. In this case, the history I-space corresponds to $\ifshist = (X \times U)^{<\Nat}$.
Let $\pi_{hist}$ be a feasible policy and let $\ifshist^{\X_h}\restriction\pi_{hist}$ be the restriction of $\ifshist^{\X_h}$ by $\pi_{hist}$ as defined in Eqn.\eqref{eqn:policy_restrict}. 
From $\pihist$, we define a new policy $\pi_{\histX} : \histX \rightarrow U$ by setting
\begin{equation}\label{eqn:policy_histY}
\pi_{\histX}(\histx) := \pihist(\imap_{\histY}\preimg(\histx)).
\end{equation}
Due to Lemmas~\ref{lem:u_for_y} and \ref{lem:refinement_kappaY}, $\imap_{\histY}$ is invertible for the domain $\ifshist^{\X_h}\restriction\pi_{hist}$ so that $\imap_{\histY}\preimg(\histx)$ is unique. In the case of a bijective sensor mapping, $\pi_{\histX}$ is a policy that maps $(x_1, \dots, x_N)$ to some $u_N$.  

% Consider the strongest sensor $h_{bij} : X \rightarrow Y$ which is a bijection. For notational simplicity, without loss of generality, we will assume $Y=X$.  
% Let $\pi_{hist}$ be a feasible policy and let $\ifshist^{\X_h}\restriction\pi_{hist}$ be the restriction of $\ifshist^{\X_h}$ by $\pi_{hist}$ as defined in Eqn.\ref{eqn:policy_restrict}. 
% From $\pihist$, we define a new policy $\pi_{\histY} : \histY \rightarrow U$ as follows
% \begin{equation}\label{eqn:policy_histY}
% \pi_{\histY}(\histy) := \pihist(\imap_{\histY}\preimg(\histy)).
% \end{equation}
% Due to Lemmas~\ref{lem:u_for_y} and \ref{lem:refinement_kappaY}, $\imap_{\histY}$ is invertible for domain $\ifshist^{\X_h}\restriction\pi_{hist}$ so that $\imap_{\histY}\preimg(\histy)$ is unique. In the case of a bijective sensor mapping, $\pi_{\histY}$ is a policy that maps $(x_1, \dots, x_N)$ to some $u_N$.  

Let $\imap_Y : \histY \rightarrow Y$ be an I-map such that $ \imap_Y(\histy)=y_N$, in which $\histy=(y_1, \dots, y_N)$ is an observation sequence of some length $N$. Considering a bijective sensor, under $\imap_Y$, $\histx \mapsto x_N$, in which $\histx=(x_1,\dots,x_N)$ for some $N$.

%\texttt{Define $\pi_X$}
%Define $\mu$ as $\mu = \pi_Y \circ h$.  

% We say that a policy $\pi_X$ is feasible if $(X,X,\id_X, \is_0)$ coupled to $(X, U, f)$ through the coupling functions $h_{bij}$ and $\pi_X$ satisfies that for any initialization $(\is_0,x)$ such that $x \in X$ the coupled system $(X \times X, \id_X*_{\pi_X,h}f)$ satisfies that the resulting trajectory will belong to $\imapb\preimg(1)$.

\begin{lemma} There exists a feasible policy $\pi_X : X \rightarrow U$ if and only if there exists a feasible $\pi_{hist}$ satisfying the following condition C1:
%\begin{enumerate}[label=({C}%{{\arabic*}})]
%\item
\[
\textrm{\textit{(C1)}} \quad\pi_{\histX}(\histx)=\pi_{\histX}(\histx') \,\, \textrm{for all} %\,\, x\in X %\,\, \textrm{and all}
\,\, 
\histx,\histx' \in \imap_Y\preimg(x).
\]
%\end{enumerate}
% for all $x\in X$ and for all $\histx,\histx' \in \imap_Y\preimg(x)$ that $\pi_{\histX}(\histx)=\pi_{\histX}(\histx')$.
\end{lemma}
\begin{proof}
\textbf{$\impliedby$ direction (Existence of $\pi_{hist}$ satisfying C1 implies existence of $\pi_X$):} Suppose there exists a feasible $\pi_{hist}$ which satisfies C1. Then, we can define a $\pi_X : X \rightarrow U$ as the policy $\pi_X=\pi_{\histX}(\histx)$, in which $\histx \in \imap_Y\preimg(x)$, since by C1, $\pi_{\histX}(\histx)=\pi_{\histX}(\histx')$ for all $\histx, \histx' \in \imap_Y\preimg(x)$. Furthermore, since $\pi_{\histX}$ is feasible $\pi_X$ is also feasible.  

\textbf{$\implies$ direction (Existence of $\pi_X$ implies existence of $\pihist$ satisfying C1)} Suppose $\pi_X$ exists. This means that for any $x \in X$ the coupled system $(X \times X, \id_X*_{\pi_X,h}f)$ initialized at $(\is_0, x)$ will result in a task accomplishing history such that $\his_N=(x_1, \pi_X(x_1), f(x_1 , \pi_X(x1)), \dots, x_N) \in $ for some $N$. Then there exists a feasible $\pi_{hist}$ such that $\pi_{hist}(\his_{k-1})=u_k$ with $u_k = \pi_X( \imap_{\histY}(\imap_Y (\his_{k-1}))$. This proves that $\pi_{hist}$ satisfies $\pi_{\histX}(\histx)=\pi_{\histX}(\histx')$.
\qed
\end{proof}

We now derive a necessary and sufficient condition for a sensor mapping $h$ to satisfy so that there exists a reactive policy $\pi_Y$ satisfying $\pi_Y \circ h = \pi_X$ for some feasible $\pi_X$. We omit the proof since it is similar to the proof of Theorem~\ref{thm:nec_suff_cond_ITS}.
\begin{theorem}\label{thm:suff_sensor} Suppose there exists a $\pi_X$.  
A sensor $h$ is sufficient for a reactive policy if and only if %there exists a feasible $\pi_X: X \rightarrow U$ and that 
$h \succeq \pi_X$ for some $\pi_X$.
\end{theorem}
\vspace{-1em}
%\begin{proof}
%Similar to Thm1... 
% \begin{itemize}
%     \item If direction: If no stationary policy, means that you can't uniquely determine which action to take at some x and so you can't have a sensor mapping from x to y that ensures y to u is unique .. if it's not a sufficient refinement then same
% \end{itemize}
%\end{proof}

\begin{corollary}
A minimal sensor $h$ for a reactive policy $\pi_X$ satisfies $h = \pi_X$.
\end{corollary}
\vspace{-1em}

The following is a direct consequence of Theorem~\ref{thm:suff_sensor} and relates to the existence result of action-based sensors when there are no crossovers (multiple actions applied at the same state) in the plan \cite{mcfassel2023reactivity}.
\vspace{-0.5em}
\begin{corollary}\label{cor:no_piX}
If there is no $\pi_X$ then there is no reactive policy.    
\end{corollary}
\vspace{-2em}

\section{Minimal Structures for Distance-Optimal Navigation}\label{sec:GNT}

In this section, we apply the results from the previous sections to the problem of distance optimal navigation in a connected polygon and characterize sufficient structures for optimal policies.

Let $X \subseteq \mathbb{R}^2$ be a connected polygon. The external system state at time $t$ is the point $x(t)=(q_x(t), q_y(t)) \in X$. 
% A path connecting two points $x_I, x_G \in P$ is a mapping $\sigma : [0,1] \rightarrow P$ such that $\sigma(0)=x_I$ and $\sigma(1)=x_G$. 
% Given a pair of initial and final points $x_I, x_G \in P$, the shortest path is the one that minimizes the total length
% \begin{equation*}\label{eqn:short_path_cost}
% L=\int_{x_I}^{x_G}
%     \sqrt{d{q}_x^2 + d{q}_y^2}.
% \end{equation*}
We assume a point robot with %system 
dynamics 
\begin{align}\label{eqn:sys_dyn_short_path}
\begin{split}
    \dot{q}_x(t) &= \sin(\theta(t) ) \\
    \dot{q}_y(t) &= \cos(\theta(t) ),
\end{split}
\end{align}
in which $\theta(t) \in S^1$ is the control input that indicates the direction under constant unit speed.
%this is equivalent to the system $\dot{x}(t) = u(t)$ with the constraint imposing constant velocity $u_1^2 + u_2^2 =c$. 
%An optimal policy $\pi_X^* : X \rightarrow U$ is the one that satisfies \texttt{Define the external system...? what are the actions? .. }
Given a pair of initial and final points $x_I, x_G \in X$, an optimal action trajectory $\theta^* : [0, T] \rightarrow S^1$ is the one that minimizes the cost function $J=\int_{0}^{T}dt$
% \begin{equation*}\label{eqn:short_path_cost}
% J=\int_{0}^{T}dt.
% \end{equation*}
and satisfies the differential equation \eqref{eqn:sys_dyn_short_path} with the initial condition $x(0)=x_I$ and final time condition $x(T)=x_G$ with $T$ being free. Note that under unit speed, $T$ is the path length.
An optimal action trajectory $\theta^*$ is one that satisfies 
\begin{equation}\label{eqn:optimal_continuous_cost}
    \frac{G^*}{\partial x_1} f_1(x, \theta^*) + \frac{G^*}{\partial x_2} f_2(x, \theta^*)= -1,
\end{equation}
in which $G^* : X \rightarrow \Re_{+}$ is the optimal cost-to-go function, i.e.~the length of the shortest path to $x_G$. The left side of Eqn.~\eqref{eqn:optimal_continuous_cost} indicates the change in the optimal cost-to-go function along the direction obtained when an optimal action is applied at $x$.
%derivative of the optimal cost-to-go function along the direction obtained when an optimal action is applied at $x$.

%\texttt{Some notational mess-up again with time parametrized states and states themselves.. also better define motion prims.}

%Let $\Sigma$ be a set of motion primitives that are constant actions such that each $\sigma \in \Sigma$ corresponds to keeping $\theta$ constant for a fixed length of time $\tau \in \Re_+$.
%The cost-to-go function $G^*$ is piecewise quadratic with level curves that are circular arcs \cite{hershberger1993efficient}, the discontinuities in the gradient direction happen at the vertices of which there are a finite number of.

%Let $U$ be a set of actions (similar to motion primitives) such that each $u\in U$ corresponds to a constant control for a finite length of time $\tau \in \Re_{+}$ such that $ u : [0,\tau] \rightarrow \bar{\theta}$ for some $\bar{\theta}\in S^1$. 
%We first show that, there exists an event-based discretization such that an optimal plan (policy) can be expressed in a stage-evolving fashion. 

Let $\V$ be the set of polygon vertices and let $U$ be a set of actions (similar to motion primitives) such that each $u \in U$ corresponds to applying constant control $\bar{\theta} \in S^1$ for a finite length of time until a point $x \in \V$ is reached. 
We first show that $U$ is sufficient to construct the shortest path between any $(x_I,x_G) \in X \times \V$. Therefore, it allows an event-based discretization such that an optimal plan (policy) can be expressed as a sequence of elements from $U$ so that the system evolves in stages.%in a stage-evolving fashion. %\texttt{Fix $x_G$ to vertices only}

\begin{lemma}\label{lem:pw_const_mps}
The optimal control trajectory $\theta^*$ resulting in the shortest path to $x_G$ from an initial state $x_I$ is obtained by applying a finite sequence of actions $(u_1, u_2, \dots, u_N) \in U^N$ for some $N$.
\end{lemma}
\begin{proof}
The cost-to-go function $G^*$ is piecewise quadratic with level curves that are circular arcs \cite{hershberger1993efficient}. Discontinuities in the gradient direction manifest only at the vertices of the polygon, of which there are finitely many.
%a finite number of. 
%Let $\V$ be the set of vertices of $X$.
Then, for any $x$, there exist $s \in \V$ and $\bar{\theta} \in S^1$ which is the direction of the gradient $\nabla_x G^*$ evaluated at $x$ such that forward integrating Eqn.\eqref{eqn:sys_dyn_short_path} with $u$ which is determined by $s$ and $\bar{\theta}$, the direction of $\nabla_x G^*$ stays constant along the respective state trajectory. Then, any optimal path can be obtained by applying a sequence of elements of $U$.
% Then there exists $\tau_i$ and $u_i : [0, \tau_i] \rightarrow \bar{\theta}_i$ in which $\bar{\theta}_i$ is the direction of the gradient $\nabla_x G^*$ evaluated at $x_i$ such that forward integrating Eqn.\eqref{eqn:sys_dyn_short_path} with $u_i$, the direction of $\nabla_x G^*$ stays constant along the respective state trajectory. 
\qed
\end{proof}
Note that this result also %intuitively 
follows from the fact that the shortest path in a polygonal environment is a sequence of bitangent edges. 
% \begin{lemma} Suppose $P$ is simply connected. Then, for every pair of $x_I, x_G \in P$, there exists a unique sequence $\histu=(u_1, \dots, u_N)$ for some $N$ resulting in the shortest path connecting $x_I$ and $x_G$.
% \end{lemma}
\begin{lemma} Suppose $X$ is simply connected. Then, for every pair of $x_I, x_G \in X$, there exists a unique sequence $\histu=(u_1, \dots, u_N)$ for some $N$ resulting in the shortest path connecting $x_I$ and $x_G$.
\end{lemma}

We will consider the gap sensor introduced in \cite{TovCohLav08} and show that the gap navigation tree as an ITS introduced therein is sufficient for supporting a set of optimal policies and it is also minimal. 

\begin{definition}[Gap sensor \cite{TovCohLav08}] Let $Y$ be the set of all cyclic sequences. A gap is defined as a discontinuity in the distance to the boundary of $X$ measured at $x$. 
A \emph{gap sensor} $h : X \rightarrow Y$ reports gaps as the cyclic order $h(x)=[g_1, g_2, \dots, g_N]$, in which each $g_i$, $i=1,\dots,N$, is a gap.
\label{defn:gap_sensor}
\end{definition}

\begin{lemma}\label{lem:finite_gap} For a path $\sigma : [0,1] \rightarrow X$ there exists a finite number of intervals $\{[t_i, t_{i+1}]\}_{i=1,\dots,N-1}$ with $t_0 =0$ and $t_N=1$ such that $h(x) = h(x')$ for all $x,x' \in [t_i, t_{i+1}]$ and all $i \in \{1, \ldots, N-1\}$ \cite{TovCohLav08}.
\end{lemma}
% \begin{proof}
% Along a path a finite number of critical visual events occur\cite{TovCohLav08}.
% \end{proof}

%\texttt{Task reach a $x_G \in \V $}
\begin{proposition} The gap sensor is not sufficient for a reactive policy. 
\end{proposition}
\begin{proof}
A sensor $h$ is called sufficient for a reactive policy if it satisfies $h \succeq \pi_X$ for some task accomplishing $\pi_X : X \rightarrow U$. Considering a simply connected polygon, there exists a unique policy $\pi_X$ which results in distance optimal navigation. Preimages of $\pi_X$ partition $X$ into an uncountable set of line segments emanating from a subset of polygon vertices. % or from $x_G$.
Then, there exist $x, x'$ with $h(x) = h(x')$ such that $\pi_X(x) \not = \pi_X(x')$. This implies $h \not\succeq \pi_X$. Therefore $h$ is not sufficient for a reactive policy. 
\qed
\end{proof}

Thanks to Lemmas~\ref{lem:pw_const_mps} and \ref{lem:finite_gap}, a policy $\pi^{\upsilon}_{hist}$ resulting in the shortest path to $x_G \in \V$ can be described over $\ifshist=(U \times Y)^{<\Nat}$ considering the history ITS. 
%which can be seen as a history \ac{Ispace} assuming sensors. 
%In the following, we will analyze sufficient structures considering an optimal policy. 
Let $\X_h = (X, U, f, h, Y)$ be the external system such that $(x, u) \mapsto \upsilon \in \V$ under $f$. 
Let $S_{hist}$ be the history ITS. 
We will consider multiple policies that correspond to distance optimal navigation to any point $\upsilon \in \V$. The set of such policies is $\{\pi^{\upsilon}_{hist}\}_{\upsilon \in \V}$. 
Each $\pi^{\upsilon}_{hist}$, once executed, corresponds to the shortest path from an initial state $x_I \in X$ to $\upsilon \in \V$.
%For each optimal policy, $S^{\X_h}_{hist}\restriction_{\pi^{x_G}_{hist}}$ 

%\texttt{Somehow explain GNT briefly... is it even possible???}

%\ac{GNT} 
Gap Navigation Trees (GNTs) are proposed in \cite{tovar2005gap}, \cite{TovCohLav08} as minimal structures for visibility and navigation related tasks. A GNT is constructed by exploring a connected planar environment, through split and merge operations. Once constructed, it encodes a portion of the shortest-map graph rooted at the current position. 
In this section, we will consider its use in distance optimal navigation in a simply connected polygon. The following describes a GNT constructed for an environment as an ITS used for distance-optimal navigation.
%such that once constructed an optimal policy can be determined. 

\begin{lemma}[GNT as an ITS] A Gap Navigation Tree (GNT) is an ITS with $S_{GNT} = (\ifs_{tree}, Y, \phi_{GNT},\is_0)$ in which $\ifs_{tree}$ is a tree, $Y$ is the set of observations of a gap sensor, and $\phi_{GNT}$ is defined through the appearance or disappearance of gaps (critical events) which correspond to observations, each of which results in a new tree $\is_{i+1}$ that is obtained from $\is_i$, by changing the root.
\label{lem:GNT_as_ITS}
\end{lemma}

\begin{proposition}
$S_{GNT}= (\ifs_{tree}, Y, \phi_{GNT},\is_0)$ supports $\pi^{\upsilon}_{hist}$ for all $\upsilon \in \V$.
\end{proposition}
\begin{sproof}
Let $S_{hist}\restriction_{\pi^\upsilon_{hist}}= (Y^{<\Nat}, Y, \phi_{\histY}, (), \pi_{\upsilon}, U^\panic)$ be the restriction of the history ITS by $\pi^{\upsilon}_{hist}$. Let $\pi^\V$ be the least upper bound of $\{ \pi_\upsilon\}_{\upsilon \in \V}$.
We need to show that there exists $\imap : Y^{<\Nat} \rightarrow \ifs_{tree}$ and satisfies $\imap \succeq \pi^{\V}$. Existence of a $\imap$ follows from Lemma~\ref{lem:GNT_as_ITS}, such that each $Y^{<\Nat}$ is mapped to $\is \in \ifs_{tree}$ by recursively applying $\phi_{GNT}$ starting from $\is_0$. Note that the domain of $\imap$ for deriving $S_{GNT}$ is restricted to attainable histories. However, since all the unattainable ones would be labeled with $\xi$ we will ignore this aspect. 
Each $\pi_\upsilon$ distinguishes observation sequences that would result from moving through an optimal sequence of vertices. 
It was shown in \cite{TovCohLav08} that following a gap results in the disappearance of the gap which happens when a vertex is reached. Therefore, $\imap$ distinguishes the same observation histories as $\pi^\V$. 
%Therefore, there exists a map $$ 
%Therefore, $\pi^\V$ distinguishes observation histories that would result from 
%$\imap \succeq \pi^{\V}$ distinguishes observation histories that correspond to 
%Given a particular observation history, $\pi$
%We need to show that $\imap$ is a refinement of $\pi^{\V}$. 
%\qed
\end{sproof}
\vspace{-0.7em}

\begin{corollary}
$S_{GNT}= (\ifs_{tree}, Y, \phi_{GNT},\is_0)$ is the minimal ITS for distance optimal navigation using a gap sensor.  
\end{corollary} 
\vspace{-1em}
% \begin{sproof}
% The preimages of $\imap$ induce the same partitioning as $\pi^\V$ so it is the minimal refinement of $\pi^\V$. Therefore, $S_{GNT}$ is a minimal ITS. 
% \end{sproof}
% \begin{sproof}
% Each $\pi^\upsilon$ labels 
% The preimages of $\imap$ induce the same partitioning as $\pi^\V$, therefore, $S_{GNT}$ is minimal. %Furthermore, each 
%\end{sproof}

% \begin{lemma}[GNT as an ITS] GNT is an ITS $(\ifs_{tree}, \Gamma, \phi_{GNT})$ such that $\phi_{GNT}$ is defined through merge and split operations as follows...
    
% \end{lemma}

% Some stuff
% \begin{itemize}
%     \item  The \ac{Imap} does not need to be metric preserving. Also \ac{Ispace} does not need to be a metric space either. 
% \end{itemize}

% \section{Discussion}
% \begin{itemize}
% \item Fix the class of policies. Notable example is sensor-feedback such that $\pi : Y \rightarrow U$ which sets $\ifs = Y$. 
% Then finding the minimal ITS relates to finding the minimal sensor sufficient for sensor feedback.. I guess it does, it might not be 
% \end{itemize}

\section{Discussion}\label{sec:discussion}

We considered solving active tasks, which requires determining an ITS and a respective policy. To this end, we have analyzed ITS and policy pairs, fixing either the particular sensor or the class of ITS characterized by a sensor mapping. For both cases, we have established the necessary and sufficient conditions for such structures to satisfy for the respective ITS to support a feasible policy. We have then applied these results to analyzing minimally sufficient structures for distance optimal navigation in the plane. 

We expect the results %presented 
to open up new avenues for research. In particular, the conditions for an ITS to satisfy is established in conjunction with a particular policy determined over histories. It is an interesting direction to characterize all such policies which in turn, would result in a characterization of all pairs of ITSs and respective policies. Defining an ordering (total or partial) over these %set of sufficient ITS and feasible policy 
pairs would allow selecting good ones, given different design objectives.
%so that depending on different design objectives, a preference can be made.

In Section~\ref{sec:GNT}, the selected actions corresponded to motion primitives. Considering the particular motion primitives, which were defined as a function of polygon vertices, a reactive sensor can be defined which establishes equivalence classes such that states at which the same direction is applied towards the same vertex belong to the same equivalence class. %Hence, the preimage of an action is the line segment emanating from a vertex and ending at a point on the polygon boundary. 
Suppose we had defined motion primitives as time (distance) parameterized functions such that each $u_i : [0, \tau_i] \rightarrow \bar{\theta}_i \in U$. Then, the resulting reactive sensor would need to distinguish also the points according to their distance to a particular vertex. This shows that there is an interesting trade-off when it comes to determining motion primitives. %We leave this interplay as an interesting future work. 

In this work, we assumed that a task-induced labeling exists based on which feasibility of policies can be checked. 
However, its existence may not be guaranteed, especially if the tasks are described over the external system states. In particular, it depends on the selection of sensors and/or motion primitives. Related to the definability of the tasks over histories, another interesting direction to follow is to establish a measure for the descriptive complexity of tasks, allowing further trade-offs for sensors and motion primitives.

\bibliographystyle{splncs04}
\bibliography{Bibliography}

\appendix

\end{document}